\documentclass[letterpaper]{article}
\usepackage{caption}
\usepackage{dirtytalk}
\usepackage{amssymb}
\usepackage{mathtools}
\usepackage{hyperref}
\hypersetup{colorlinks=true,linkcolor=blue}
\usepackage{pifont}
\usepackage{tikz}
\usetikzlibrary{calc}
\usepackage{booktabs}
\usepackage{adjustbox}

\title{Proving Conjectures Acquired\\ by Composing Multiple Biases}

\author{
Jovial Cheukam-Ngouonou\textsuperscript{\rm 1, 2, 4},
Ramiz Gindullin\textsuperscript{\rm 1, 2},\\
Nicolas Beldiceanu\textsuperscript{\rm 1, 2},
R\'emi Douence\textsuperscript{\rm 1, 2, 3},
Claude-Guy Quimper\textsuperscript{\rm 4}
}

\date{%
    \small{
    \textsuperscript{\rm 1}IMT Atlantique, Nantes, France\\
    \textsuperscript{\rm 2}LS2N, Nantes, France,\\
    \textsuperscript{\rm 3}INRIA, Nantes, France,\\
    \textsuperscript{\rm 4}Université Laval, Quebec City, Canada}
}


\newtheorem{conjecture}{Conjecture}
\newtheorem{example}{Example}
\newtheorem{proof}{Proof}

\renewenvironment{proof}[1][\proofname]{\par
    \textbf{#1.}
}{%
    $\blacksquare$
}

\newcommand{\nv}{v}
\newcommand{\nc}{c}
\newcommand{\ns}{s}
\newcommand{\minc}{\underline{c}}
\newcommand{\maxc}{\overline{c}}
\newcommand{\mins}{\underline{s}}
\newcommand{\maxs}{\overline{s}}
\newcommand{\nval}{\mathit{nval}}
\newcommand{\pn}{n}
\newcommand{\pmin}{\underline{m}}
\newcommand{\pmax}{\overline{m}}
\newcommand{\prange}{\underline{\overline{m}}}
\newcommand{\tn}{n}
\newcommand{\tleaves}{\ell}
\newcommand{\tmindegree}{\underline{d}}
\newcommand{\tmaxdegree}{\overline{d}}

\begin{document}
\maketitle

\begin{abstract}
We present the proofs of the conjectures mentioned in the paper published in the proceedings of the 2024 AAAI conference~\cite{DecompGindullinBCDQ24}, and discovered by the decomposition methods presented in the same paper.
\end{abstract}

\section{Introduction}

Four conjectures on {\sc Digraphs} are described and proved in Section~\ref{sec:graph}.
Then, in Section~\ref{sec:tree}, we present a conjecture for {\sc Rooted Trees} and prove it.

\section{Proving Four Conjectures on {\sc Digraphs} Found Using Decomposition Methods}
\label{sec:graph}

This section presents four conjectures on {\sc Digraphs} where every vertex is adjacent to at least an arc.

\paragraph*{Notations}
The expression $(\mathit{cond}\,?\,x:y)$ denotes $x$ if condition $\mathit{cond}$ holds, $y$ otherwise. 
We also introduce the notation $\left[ e \right]$ defined as $(\mathit{e}\,?\,1:0)$.

\begin{conjecture}\label{ex:conjecture1}
Conj.~(\ref{eq:conjecture1}) provides a
sharp lower bound on the number of connected components $c$ of a digraph wrt its number of vertices $v$, the greatest number of vertices $\maxc$ inside a connected component, and the smallest number of vertices $\mins$ in a strongly connected component (scc).
\begin{equation}\label{eq:conjecture1}
\nc\geq
\textstyle\lceil\nv\slash\maxc\rceil+
\textstyle\left[\neg ((2\cdot\mins\leq\maxc) \lor (\mins\geq (\nv \bmod \maxc=0\;?\;\maxc\;:\nv\bmod\maxc))\right]
\end{equation}
\end{conjecture}

\begin{conjecture}\label{ex:conjecture2}
Conj.~(\ref{eq:conjecture2}) gives a
sharp lower bound $\nc$ of a digraph wrt its number of scc $\ns$, and the $\maxc$ and $\mins$ characteristics introduced in Conj.~(\ref{eq:conjecture1}).
\begin{equation}\label{eq:conjecture2}
\nc\geq \color{black}\biggl\lceil\color{black}\textstyle \color{black}\dfrac{\ns}{
\left\lfloor\maxc/\mins\right\rfloor}
\color{black}\biggr\rceil
\color{black}
\end{equation}

\end{conjecture}

\begin{conjecture}\label{ex:conjecture3}
Conj.~(\ref{eq:conjecture3}) depicts a
sharp lower bound on the maximum number of vertices $\maxs$ inside an scc of a digraph wrt the $\ns$ and $\mins$ characteristics previously in\-troduced.
\begin{equation}\label{eq:conjecture3}
\maxs \geq \biggl\lceil \left(\nv=\mins\; ?\; \nv : \nv-\mins\right) / \left(\ns-1 + [\ns = 1]\right)
 \biggr\rceil
\end{equation}
\end{conjecture}

\begin{conjecture}\label{ex:conjecture4}
Conj.~(\ref{eq:conjecture4}) depicts a
sharp lower bound on the maximum number of vertices $\maxc$ inside a connected component of a digraph wrt $\minc$ (the smallest
number of vertices inside a connected component) and the $\nv$, $\nc$ and $\maxs$ characteristics previously introduced.
\begin{equation}\label{eq:conjecture4}
    \maxc\geq \color{black}\bigg(\color{black}\color{black}v=c\cdot\minc
    \hspace{0.2cm}
    ?
    \hspace{0.2cm}
   \minc \hspace{0.1cm}:\hspace{0.1cm}\max\left(\maxs,\left\lceil(v-\minc)/(c-1)\right\rceil\right)\bigg)
    \end{equation}
\end{conjecture}

We now give the proofs for the correctness of the four bounds (\ref{ex:conjecture1})--(\ref{ex:conjecture4}). Although we do not prove that these bounds are sharp, 
we show for each proof by case that each case can indeed be reached by at least one example.
Most of the proofs go along this line.
\begin{itemize}
\item
First, we \say{unfold}~the conjecture, i.e.~we remove all its explicit or implicit conditions to create a conjecture whose cases are mutually disjoint. Each case consists of a precondition, and a corresponding simplified version of the conjecture assuming the precondition holds. The term \emph{implicit condition} refers to a condition derived from the simplification of a $\min$ (resp.~$\max$) term by assuming which part is smaller (resp.~greater) than the other part.
\item
Second, we usually prove by contradiction each case separately.
\end{itemize}
Tables~(\ref{tab:conj1}), ~(\ref{tab:conj3}), and~(\ref{tab:conj4}) 
resp. show the cases attached to Conj.~(\ref{ex:conjecture1}), (\ref{ex:conjecture3}), and~(\ref{ex:conjecture4}). For each case, the last column of tables~\ref{tab:conj1}, \ref{tab:conj3}, and~\ref{tab:conj4} gives an example of a digraph that reaches the corresponding bound. Note that, Conj.~(\ref{ex:conjecture2}) has a single case and does not require to partition the formula.\\

\begin{proof}[Proof of Conj.~(\ref{ex:conjecture1})]
We associate three conditions
\ding{172} $v\bmod\maxc=0$,
\ding{173} $\mins\geq (v \bmod \maxc=0\;?\;\maxc\;:\;v\bmod\maxc)$,
\ding{174} $2\cdot\mins\leq\maxc$
to Conj.~(\ref{ex:conjecture1}) and combine them in a systematic way to produce a set of seven mutually incompatible cases where the formula can be simplified as shown in Table~\ref{tab:conj1}. Note that the case $\mbox{\ding{172}}\land\mbox{\ding{173}}\land\mbox{\ding{174}}$ is not feasible.
We now prove each of the seven incompatible cases by contradiction.

\noindent\textbf{Proof by contradiction of case $\mbox{\ding{172}}\land\mbox{\ding{173}}\land\neg\mbox{\ding{174}}$ and case $\mbox{\ding{172}}\land\neg\mbox{\ding{173}}\land\mbox{\ding{174}}$}:
For all digraphs we have $c\cdot\maxc\geq v$. But the negation of Conj.~(\ref{ex:conjecture1}), namely $c<\left\lfloor v/\maxc\right\rfloor$, implies that $c\cdot\maxc < v$, a contradiction, as the implication assumes that $x < \lfloor y/z\rfloor\Rightarrow x\cdot z < y$.

\noindent\textbf{Proof by contradiction of cases $\neg\mbox{\ding{172}}\land\mbox{\ding{173}}\land\mbox{\ding{174}}$, $\neg\mbox{\ding{172}}\land\mbox{\ding{173}}\land\neg\mbox{\ding{174}}$}, and $\neg\mbox{\ding{172}}\land\neg\mbox{\ding{173}}\land\mbox{\ding{174}}$:
For any digraph we have that $c\cdot\maxc\geq v$, which implies that $c\geq\left\lfloor v/\maxc\right\rfloor$.
From the negation of Conj.~(\ref{ex:conjecture1}) we get
$c<\left\lfloor v/\maxc\right\rfloor+1$, that is $c\leq\left\lfloor v/\maxc\right\rfloor$.
So finally we have $c = \left\lfloor v/\maxc\right\rfloor$, and the corresponding maximum number of vertices is $\left\lfloor v/\maxc\right\rfloor\cdot\maxc$,
which is strictly less than $v=\left\lfloor v/\maxc\right\rfloor\cdot\maxc+v\bmod\maxc$ as, from the hypothesis $\neg$\ding{172}, we have $v\bmod\maxc>0$.

\noindent\textbf{Proof by contradiction of case $\mbox{\ding{172}}\land\neg\mbox{\ding{173}}\land\neg\mbox{\ding{174}}$:}
We derive from the negation of Conj.~(\ref{ex:conjecture1}) a lower bound and an upper bound on the number of vertices $v$ and show that the lower bound exceeds the upper bound.

\noindent\textbf{[Computing of a lower bound of $v$]}
By simplifying the negation of the conjecture, namely $c<\left\lfloor v/\maxc\right\rfloor+1$, we obtain $c\leq\left\lfloor v/\maxc\right\rfloor$, which implies that $c\cdot\maxc\leq v$.

\noindent\textbf{[Computing of an upper bound of $v$]}
From $\neg$\ding{174} we have that each connected component contains exactly one scc. One of these connected components has size $\mins$, while let $|c_i|$ denote the size of each of the remaining connected components (with $i\in[1,c-1]$). Consequently the total number of vertices is $\mins+\sum_{i=1}^{c-1}|c_i|$, which is less than or equal to $\mins+(c-1)\cdot\maxc$.

\noindent\textbf{[Deriving a contradiction]}
We showed that $v\in[c\cdot\maxc,\mins+(c-1)\cdot\maxc]$.
Hence we should have $c\cdot\maxc\leq\mins+(c-1)\cdot\maxc$.
But \ding{172} implies that $\neg$\ding{173} can be simplified as $\mins<\maxc$, which leads to the fact that the upper bound $\mins+(c-1)\cdot\maxc$ is strictly less than $c\cdot\maxc$, a contradiction.

\noindent\textbf{Proof by contradiction of case $\neg\mbox{\ding{172}}\land\neg\mbox{\ding{173}}\land\neg\mbox{\ding{174}}$:}
For any digraph we have that $c\cdot\maxc\geq v$.
Moreover, $v$ can be decomposed wrt $\maxc$ in the following way
$v = \left\lfloor v/\maxc\right\rfloor\cdot\maxc + v\bmod\maxc$.
This leads to the following inequality
$c\cdot\maxc\geq \left\lfloor v/\maxc\right\rfloor\cdot\maxc + v\bmod\maxc$,
which can be rewritten as
$c\geq \left\lfloor v/\maxc\right\rfloor + (v\bmod\maxc)/\maxc$.
From $\neg$\ding{172}, i.e. $v\bmod\maxc>0$,
and therefore
$c>\left\lfloor v/\maxc\right\rfloor$,
i.e.~$c\geq\left\lfloor v/\maxc\right\rfloor+1$.
By considering the negation of Conj.~(\ref{ex:conjecture1})
we obtain $c=\left\lfloor v/\maxc\right\rfloor+1$,
and the corresponding maximum number of vertices is
$\left\lfloor v/\maxc\right\rfloor\cdot\maxc+\mins$, as
one of the connected components must have $\mins$ vertices since, from hypothesis $\neg$\ding{174} each connected component contains exactly one scc.
From hypothesis~$\neg$\ding{173}, $\left\lfloor v/\maxc\right\rfloor\cdot\maxc+\mins$ is strictly less than
$v=\left\lfloor v/\maxc\right\rfloor\cdot\maxc + v\bmod\maxc$,
a contradiction.
\end{proof}\\

\begin{proof}[Proof of Conj.~(\ref{ex:conjecture2})]
First, observe that the maximum number of scc within a same connected component is $\left\lfloor\maxc/\mins\right\rfloor$.
Second, as we have to dispatch $s$ scc in the connected components of the digraph, and as each connected component contains at most $\left\lfloor\maxc/\mins\right\rfloor$ scc, we obtain the bound given by Conj.~(\ref{eq:conjecture2}), namely
$\left\lceil s/\!\left\lfloor\maxc/\mins\right\rfloor\right\rceil$.
\end{proof}\\

\begin{proof}[Proof of Conj.~(\ref{ex:conjecture3})] The proof is done by decomposing a digraph as a partition of scc.
In Table~\ref{tab:conj3}, we associate the two conditions
\ding{175} $v=\mins$,
\ding{176} $s=1$
to Conj.~(\ref{ex:conjecture3}) and combine them in a systematic way to produce a set of mutually incompatible cases where the formula can be simplified.
Note that cases $\mbox{\ding{175}}\land\neg\mbox{\ding{176}}$ and $\neg\mbox{\ding{175}}\land\mbox{\ding{176}}$ are absent as they are not feasible.

\noindent\textbf{Case $\mbox{\ding{175}}\land\mbox{\ding{176}}$:}
This case is obvious as Condition~\ding{176} indicates that the digraph has just one scc,
meaning that it also corresponds to the largest and smallest scc.

\noindent\textbf{Case $\neg\mbox{\ding{175}}\land\neg\mbox{\ding{176}}$:}
Now, as $s>1$, let $s_i$ be the number of vertices of the $i$-th scc of the digraph (with $i\in [1,s-1]$.
We have $\sum_{i=1}^{s-1}\maxs\geq\sum_{i=1}^{s-1}s_i$, which is equivalent to $\maxs\cdot(s-1)\geq v-\mins$.
Therefore $\maxs\geq (v-\mins)/(s-1)$, and as $\maxs\in \mathbb{N}$ we obtain $\maxs\geq\left\lceil (v-\mins)/(s-1)\right\rceil$.
\end{proof}\\

\begin{proof}[Proof of Conj.~(\ref{ex:conjecture4})] The proof is done by decomposing a digraph as a partition of connected components. We associate the following two conditions
\ding{177} $v=c\cdot\minc$,
\ding{178} $\maxs\geq\left\lceil (v-\minc)/(c-1)\right\rceil$
to Conj.~(\ref{ex:conjecture4}) and combine them in a systematic way to produce a set of mutually incompatible cases where the formula can be simplified, as shown in Table~\ref{tab:conj4}.

\noindent\textbf{[Case $\mbox{\ding{177}}\land(\mbox{\ding{178}}\lor\neg\mbox{\ding{178}})$]} As condition~\ding{177} indicates that every connected component of the digraph has the size of the smallest connected component, it implies that $\maxc = \minc$.

\noindent\textbf{[Cases $\neg\mbox{\ding{177}}\land\mbox{\ding{178}}$ and $\neg\mbox{\ding{177}}\land\neg\mbox{\ding{178}}$]} Let $|c_i|$ denotes the number of vertices of the $i$-th connected component of a digraph (with $i\in[1,c]$). First note that for all digraphs, as $\sum_{i=1}^{c-1}\maxc\geq\sum_{i=1}^{c-1}|c_i|$, we have $\maxc\cdot(c-1)\geq\nv-\minc$. And because $\neg$\ding{177} implies that $c>1$, we have then $\maxc\geq\left\lceil(v-\minc)/(c-1)\right\rceil$. Therefore, as for all digraphs we also have $\maxc \geq \maxs$, we derive that $\maxc \geq \max\left(\maxs,\left\lceil (v-\minc)/(c-1)\right\rceil\right)$, which proves the two cases $\neg\mbox{\ding{177}}\land\mbox{\ding{178}}$ and $\neg\mbox{\ding{177}}\land\neg\mbox{\ding{178}}$.
\end{proof}

\tikzset{every loop/.style={min distance=3.5mm,in=120,out=60,looseness=4,->,brown}}
\begin{table*}[bt]
\center
\begin{tabular}{ccccc}
\toprule
{\footnotesize Cond.~\ding{172}}           & {\footnotesize Cond.~\ding{173}}           & {\footnotesize Cond.~\ding{174}}           & {\footnotesize simplified  bound for $c$}                  & {\footnotesize witness parameters}                                                                                                    \\
\midrule
{\footnotesize\phantom{$\neg$}\ding{172}} & {\footnotesize\phantom{$\neg$}\ding{173}} & {\footnotesize$\neg$\ding{174}}           & {\footnotesize$\left\lfloor v/\,\maxc\right\rfloor\hspace*{3pt}\phantom{+1}$}
&
\adjustbox{valign=m}{\begin{tikzpicture}
\begin{scope}
\node[circle,draw=brown,fill=brown,minimum size=0pt,inner sep=0pt] (c11) at (1,0.3) {\color{brown}.};
\path
    (c11) edge [loop] node {} (c11);
\node[anchor=west] (a1) at (-1.25,0.3) {
\scriptsize$\left\{\hspace*{-0.3pt}\begin{array}{lr}  v=&1\\
                       \maxc=&1\\
                       \mins=&1
       \end{array}\right.$};
\node[anchor=west] (a2) at (3,0.3) {\scriptsize$c=1$};
\end{scope}
\end{tikzpicture}}
\\
{\footnotesize\phantom{$\neg$}\ding{172}} & {\footnotesize$\neg$\ding{173}}           & {\footnotesize\phantom{$\neg$}\ding{174}} & {\footnotesize$\left\lfloor v/\,\maxc\right\rfloor\hspace*{3pt}\phantom{+1}$}
&
\adjustbox{valign=m}{\begin{tikzpicture}
\begin{scope}
\node[circle,draw=brown,fill=brown,minimum size=0pt,inner sep=0pt] (c11) at (1,0.15) {\color{brown}.};
\node[circle,draw=brown,fill=brown,minimum size=0pt,inner sep=0pt] (c12) at (1,0.45) {\color{brown}.};
\draw[->,brown] (c11) edge (c12);
\node[anchor=west] (a1) at (-1.25,0.3) {
\scriptsize$\left\{\hspace*{-0.3pt}\begin{array}{lr}  v=&2\\
                       \maxc=&2\\
                       \mins=&1
       \end{array}\right.$};
\node[anchor=west] (a2) at (3,0.3) {\scriptsize$c=1$};
\end{scope}
\end{tikzpicture}}
\\
{\footnotesize\phantom{$\neg$}\ding{172}} & {\footnotesize$\neg$\ding{173}}           & {\footnotesize$\neg$\ding{174}}           & {\footnotesize$\left\lfloor v/\,\maxc\right\rfloor+1$}           &
\adjustbox{valign=m}{\begin{tikzpicture}
\begin{scope}
\node[circle,draw=brown,fill=brown,minimum size=0pt,inner sep=0pt] (c11) at (1,0) {\color{brown}.};
\node[circle,draw=brown,fill=brown,minimum size=0pt,inner sep=0pt] (c12) at (1,0.3) {\color{brown}.};
\node[circle,draw=brown,fill=brown,minimum size=0pt,inner sep=0pt] (c13) at (1,0.6) {\color{brown}.};
\draw[->,brown] (c11) edge (c12);
\draw[->,brown] (c12) edge (c13);
\draw[->,brown,rounded corners=2pt] (c13) -- ($(c13)+(0.2,0)$) -- ($(c11)+(0.2,0)$) -- (c11);
\node[anchor=west] (a1) at (-1.25,0.3) {
\scriptsize$\left\{\hspace*{-0.3pt}\begin{array}{lr}  v=&9\\
                       \maxc=&3\\
                       \mins=&2
       \end{array}\right.$};
\node[anchor=west] (a2) at (3,0.3) {\scriptsize$c=4$};
\end{scope}
\begin{scope}[xshift=0.6cm]
\node[circle,draw=brown,fill=brown,minimum size=0pt,inner sep=0pt] (c11) at (1,0) {\color{brown}.};
\node[circle,draw=brown,fill=brown,minimum size=0pt,inner sep=0pt] (c12) at (1,0.3) {\color{brown}.};
\draw[->,brown] (c11) edge (c12);
\draw[->,brown,rounded corners=2pt] (c12) -- ($(c12)+(0.2,0)$) -- ($(c11)+(0.2,0)$) -- (c11);
\end{scope}
\begin{scope}[xshift=1.05cm]
\node[circle,draw=brown,fill=brown,minimum size=0pt,inner sep=0pt] (c11) at (1,0) {\color{brown}.};
\node[circle,draw=brown,fill=brown,minimum size=0pt,inner sep=0pt] (c12) at (1,0.3) {\color{brown}.};
\draw[->,brown] (c11) edge (c12);
\draw[->,brown,rounded corners=2pt] (c12) -- ($(c12)+(0.2,0)$) -- ($(c11)+(0.2,0)$) -- (c11);
\end{scope}
\begin{scope}[xshift=1.5cm]
\node[circle,draw=brown,fill=brown,minimum size=0pt,inner sep=0pt] (c11) at (1,0) {\color{brown}.};
\node[circle,draw=brown,fill=brown,minimum size=0pt,inner sep=0pt] (c12) at (1,0.3) {\color{brown}.};
\draw[->,brown] (c11) edge (c12);
\draw[->,brown,rounded corners=2pt] (c12) -- ($(c12)+(0.2,0)$) -- ($(c11)+(0.2,0)$) -- (c11);
\end{scope}
\end{tikzpicture}}
\\
{\footnotesize$\neg$\ding{172}}           & {\footnotesize\phantom{$\neg$}\ding{173}} & {\footnotesize\phantom{$\neg$}\ding{174}} & {\footnotesize$\left\lfloor v/\,\maxc\right\rfloor+1$}           &
\adjustbox{valign=m}{\begin{tikzpicture}
\begin{scope}
\node[circle,draw=brown,fill=brown,minimum size=0pt,inner sep=0pt] (c11) at (1,0.15) {\color{brown}.};
\node[circle,draw=brown,fill=brown,minimum size=0pt,inner sep=0pt] (c12) at (1,0.45) {\color{brown}.};
\draw[->,brown] (c11) edge (c12);
\node[anchor=west] (a1) at (-1.25,0.3) {
\scriptsize$\left\{\hspace*{-0.3pt}\begin{array}{lr}  v=&3\\
                       \maxc=&2\\
                       \mins=&1
       \end{array}\right.$};
\node[anchor=west] (a2) at (3,0.3) {\scriptsize$c=2$};
\end{scope}
\begin{scope}[xshift=0.6cm]
\node[circle,draw=brown,fill=brown,minimum size=0pt,inner sep=0pt] (c11) at (1,0.15) {\color{brown}.};
\path
    (c11) edge [loop] node {} (c11);
\end{scope}
\end{tikzpicture}}
\\
{\footnotesize$\neg$\ding{172}}           & {\footnotesize\phantom{$\neg$}\ding{173}} & {\footnotesize$\neg$\ding{174}}           & {\footnotesize$\left\lfloor v/\,\maxc\right\rfloor+1$}           &
\adjustbox{valign=m}{\begin{tikzpicture}
\begin{scope}
\node[circle,draw=brown,fill=brown,minimum size=0pt,inner sep=0pt] (c11) at (1,0) {\color{brown}.};
\node[circle,draw=brown,fill=brown,minimum size=0pt,inner sep=0pt] (c12) at (1,0.3) {\color{brown}.};
\node[circle,draw=brown,fill=brown,minimum size=0pt,inner sep=0pt] (c13) at (1,0.6) {\color{brown}.};
\draw[->,brown] (c11) edge (c12);
\draw[->,brown] (c12) edge (c13);
\draw[->,brown,rounded corners=2pt] (c13) -- ($(c13)+(0.2,0)$) -- ($(c11)+(0.2,0)$) -- (c11);
\node[anchor=west] (a1) at (-1.25,0.3) {
\scriptsize$\left\{\hspace*{-0.3pt}\begin{array}{lr}  v=&5\\
                       \maxc=&3\\
                       \mins=&2
       \end{array}\right.$};
\node[anchor=west] (a2) at (3,0.3) {\scriptsize$c=2$};
\end{scope}
\begin{scope}[xshift=0.6cm]
\node[circle,draw=brown,fill=brown,minimum size=0pt,inner sep=0pt] (c11) at (1,0) {\color{brown}.};
\node[circle,draw=brown,fill=brown,minimum size=0pt,inner sep=0pt] (c12) at (1,0.3) {\color{brown}.};
\draw[->,brown] (c11) edge (c12);
\draw[->,brown,rounded corners=2pt] (c12) -- ($(c12)+(0.2,0)$) -- ($(c11)+(0.2,0)$) -- (c11);
\end{scope}
\end{tikzpicture}}
\\
{\footnotesize$\neg$\ding{172}}           & {\footnotesize$\neg$\ding{173}}           & {\footnotesize\phantom{$\neg$}\ding{174}} & {\footnotesize$\left\lfloor v/\,\maxc\right\rfloor+1$}           &
\adjustbox{valign=m}{\begin{tikzpicture}
\begin{scope}
\node[circle,draw=brown,fill=brown,minimum size=0pt,inner sep=0pt] (c11) at (1,0) {\color{brown}.};
\node[circle,draw=brown,fill=brown,minimum size=0pt,inner sep=0pt] (c12) at (1,0.3) {\color{brown}.};
\node[circle,draw=brown,fill=brown,minimum size=0pt,inner sep=0pt] (c13) at (1,0.6) {\color{brown}.};
\draw[->,brown] (c11) edge (c12);
\draw[->,brown] (c12) edge (c13);
\node[anchor=west] (a1) at (-1.25,0.3) {
\scriptsize$\left\{\hspace*{-0.3pt}\begin{array}{lr}  v=&5\\
                       \maxc=&3\\
                       \mins=&1
       \end{array}\right.$};
\node[anchor=west] (a2) at (3,0.3) {\scriptsize$c=2$};
\end{scope}
\begin{scope}[xshift=0.6cm]
\node[circle,draw=brown,fill=brown,minimum size=0pt,inner sep=0pt] (c11) at (1,0) {\color{brown}.};
\node[circle,draw=brown,fill=brown,minimum size=0pt,inner sep=0pt] (c12) at (1,0.3) {\color{brown}.};
\draw[->,brown] (c11) edge (c12);
\end{scope}
\end{tikzpicture}}
\\
{\footnotesize$\neg$\ding{172}}           & {\footnotesize$\neg$\ding{173}}           & {\footnotesize$\neg$\ding{174}}           & {\footnotesize$\left\lfloor v/\,\maxc\right\rfloor+2$}           & 
\adjustbox{valign=m}{\begin{tikzpicture}
\begin{scope}
\node[circle,draw=brown,fill=brown,minimum size=0pt,inner sep=0pt] (c11) at (1,0) {\color{brown}.};
\node[circle,draw=brown,fill=brown,minimum size=0pt,inner sep=0pt] (c12) at (1,0.3) {\color{brown}.};
\node[circle,draw=brown,fill=brown,minimum size=0pt,inner sep=0pt] (c13) at (1,0.6) {\color{brown}.};
\node[circle,draw=brown,fill=brown,minimum size=0pt,inner sep=0pt] (c14) at (1.3,0.6) {\color{brown}.};
\node[circle,draw=brown,fill=brown,minimum size=0pt,inner sep=0pt] (c15) at (1.3,0) {\color{brown}.};
\draw[->,brown] (c11) edge (c12);
\draw[->,brown] (c12) edge (c13);
\draw[->,brown] (c13) edge (c14);
\draw[->,brown] (c14) edge (c15);
\draw[->,brown] (c15) edge (c11);
\node[anchor=west] (a1) at (-1.25,0.3) {
\scriptsize$\left\{\hspace*{-0.3pt}\begin{array}{lr}v=&14\\
                       \maxc=&5\\
                       \mins=&3
       \end{array}\right.$};
\node[anchor=west] (a2) at (3,0.3) {\scriptsize$c=4$};
\end{scope}
\begin{scope}[xshift=0.6cm]
\node[circle,draw=brown,fill=brown,minimum size=0pt,inner sep=0pt] (c11) at (1,0) {\color{brown}.};
\node[circle,draw=brown,fill=brown,minimum size=0pt,inner sep=0pt] (c12) at (1,0.3) {\color{brown}.};
\node[circle,draw=brown,fill=brown,minimum size=0pt,inner sep=0pt] (c13) at (1,0.6) {\color{brown}.};
\draw[->,brown] (c11) edge (c12);
\draw[->,brown] (c12) edge (c13);
\draw[->,brown,rounded corners=2pt] (c13) -- ($(c13)+(0.2,0)$) -- ($(c11)+(0.2,0)$) -- (c11);
\end{scope}
\begin{scope}[xshift=1.05cm]
\node[circle,draw=brown,fill=brown,minimum size=0pt,inner sep=0pt] (c11) at (1,0) {\color{brown}.};
\node[circle,draw=brown,fill=brown,minimum size=0pt,inner sep=0pt] (c12) at (1,0.3) {\color{brown}.};
\node[circle,draw=brown,fill=brown,minimum size=0pt,inner sep=0pt] (c13) at (1,0.6) {\color{brown}.};
\draw[->,brown] (c11) edge (c12);
\draw[->,brown] (c12) edge (c13);
\draw[->,brown,rounded corners=2pt] (c13) -- ($(c13)+(0.2,0)$) -- ($(c11)+(0.2,0)$) -- (c11);
\end{scope}
\begin{scope}[xshift=1.5cm]
\node[circle,draw=brown,fill=brown,minimum size=0pt,inner sep=0pt] (c11) at (1,0) {\color{brown}.};
\node[circle,draw=brown,fill=brown,minimum size=0pt,inner sep=0pt] (c12) at (1,0.3) {\color{brown}.};
\node[circle,draw=brown,fill=brown,minimum size=0pt,inner sep=0pt] (c13) at (1,0.6) {\color{brown}.};
\draw[->,brown] (c11) edge (c12);
\draw[->,brown] (c12) edge (c13);
\draw[->,brown,rounded corners=2pt] (c13) -- ($(c13)+(0.2,0)$) -- ($(c11)+(0.2,0)$) -- (c11);
\end{scope}
\end{tikzpicture}}
\\
\bottomrule
\end{tabular}
\caption{Simplified versions of Conj.~(\ref{ex:conjecture1}) wrt conditions \ding{172}, \ding{173} and \ding{174}}
\label{tab:conj1} 
\end{table*}

\tikzset{every loop/.style={min distance=3.5mm,in=330,out=30,looseness=4,->,brown}}
\begin{table*}[bt]
\center
\begin{tabular}{cccc}
\toprule
{\footnotesize Cond.~\ding{175}}                          & {\footnotesize Cond.~\ding{176}}                          & {\footnotesize simplified bound for $\maxs$}                                            & {\footnotesize witness parameters}                                \\
\midrule
{\footnotesize\phantom{$\neg$}\ding{175}} & {\footnotesize\phantom{$\neg$}\ding{176}} & {\footnotesize$v$}                                          &
\adjustbox{valign=m}{\begin{tikzpicture}
\begin{scope}
\node[circle,draw=brown,fill=brown,minimum size=0pt,inner sep=0pt] (c11) at (1.7,0) {\color{brown}.};
\path
    (c11) edge [loop] node {} (c11);
\node[anchor=west] (a1) at (-1.25,0) {$v=1$ $s=1$ $\mins=1$};
\node[anchor=west] (a1) at (2.2,0) {$\maxs=1$};
\end{scope}
\end{tikzpicture}}
\\
{\footnotesize$\neg$\ding{175}}           & {\footnotesize$\neg$\ding{176}}           &
{\footnotesize$\left\lceil(v-\mins)/(s-1)\right\rceil$} &
\adjustbox{valign=m}{\begin{tikzpicture}
\begin{scope}
\node[circle,draw=brown,fill=brown,minimum size=0pt,inner sep=0pt] (c11) at (1.7,0) {\color{brown}.};
\node[circle,draw=brown,fill=brown,minimum size=0pt,inner sep=0pt] (c12) at (2.0,0) {\color{brown}.};
\draw[->,brown] (c11) edge (c12);
\node[anchor=west] (a1) at (-1.25,0) {$v=2$ $s=2$ $\mins=1$};
\node[anchor=west] (a1) at (2.2,0) {$\maxs=1$};
\end{scope}
\end{tikzpicture}}
\\
\bottomrule
\end{tabular}
\caption{Simplified versions of Conj.~(\ref{ex:conjecture3}) wrt conditions \ding{175} and \ding{176}}
\label{tab:conj3} 
\end{table*}

\tikzset{every loop/.style={min distance=3.5mm,in=330,out=30,looseness=4,->,brown}}
\begin{table*}[bt]
\center
\begin{tabular}{cccc}
\toprule
{\footnotesize Cond.~\ding{177}}                          & {\footnotesize Cond.~\ding{178}}                          & {\footnotesize simplified bound for $\maxc$}                                            & {\footnotesize witness parameters}                                \\
\midrule
{\footnotesize\phantom{$\neg$}\ding{177}}           & {\footnotesize\ding{178}$\lor\neg$\ding{178}}           & {\footnotesize$\minc$} &
\adjustbox{valign=m}{\begin{tikzpicture}
\begin{scope}
\node[circle,draw=brown,fill=brown,minimum size=0pt,inner sep=0pt] (c11) at (1.7,0.15) {\color{brown}.};
\path
    (c11) edge [loop] node {} (c11);
\node[anchor=west] (a1) at (-1.25,0.3) {
\scriptsize$\left\{\hspace*{-0.3pt}\begin{array}{ll}  v=1 & \minc=1\\
 c=1 & \maxs=1
       \end{array}\right.$};
\node[anchor=west] (a2) at (3,0.3) {\scriptsize$\maxc=1$};
\end{scope}
\end{tikzpicture}}
\\
{\footnotesize$\neg$\ding{177}} & {\footnotesize\ding{178}\phantom{$\lor\neg$\ding{178}}} & {\footnotesize$\maxs$}                                          &
\adjustbox{valign=m}{\begin{tikzpicture}
\begin{scope}
\node[circle,draw=brown,fill=brown,minimum size=0pt,inner sep=0pt] (c11) at (1.7,0.15) {\color{brown}.};
\node[circle,draw=brown,fill=brown,minimum size=0pt,inner sep=0pt] (c12) at (1.7,0.45) {\color{brown}.};
\node[circle,draw=brown,fill=brown,minimum size=0pt,inner sep=0pt] (c13) at (2.0,0.45) {\color{brown}.};
\draw[->,brown] (c11) edge (c12);
\draw[->,brown] (c12) edge (c13);
\draw[->,brown,rounded corners=2pt] (c13) -- ($(c13)-(0,0.3)$) -- ($(c11)+(0.2,0)$) -- (c11);
\node[anchor=west] (a1) at (-1.25,0.3) {
\scriptsize$\left\{\hspace*{-0.3pt}\begin{array}{ll}  v=5 & \minc=1\\
 c=3 & \maxs=3
       \end{array}\right.$};
\node[anchor=west] (a2) at (3,0.3) {\scriptsize$\maxc=3$};
\end{scope}
\begin{scope}[xshift=0.6cm]
\node[circle,draw=brown,fill=brown,minimum size=0pt,inner sep=0pt] (c11) at (1.7,0.15) {\color{brown}.};
\node[circle,draw=brown,fill=brown,minimum size=0pt,inner sep=0pt] (c12) at (1.7,0.45) {\color{brown}.};
\path
    (c11) edge [loop] node {} (c11);
\path
    (c12) edge [loop] node {} (c12);
\end{scope}

\end{tikzpicture}}
\\
{\footnotesize$\neg$\ding{177}} & {\phantom{\footnotesize\ding{178}$\lor$}$\neg$\ding{178}} & 
{\footnotesize$\left\lceil(v-\minc)/(c-1)\right\rceil$} &
\adjustbox{valign=m}{\begin{tikzpicture}
\begin{scope}
\node[circle,draw=brown,fill=brown,minimum size=0pt,inner sep=0pt] (c11) at (1.7,0.15) {\color{brown}.};
\node[circle,draw=brown,fill=brown,minimum size=0pt,inner sep=0pt] (c12) at (1.7,0.45) {\color{brown}.};
\draw[->,brown] (c11) edge (c12);
\node[anchor=west] (a1) at (-1.25,0.3) {
\scriptsize$\left\{\hspace*{-0.3pt}\begin{array}{ll}  v=3 & \minc=1\\
 c=2 & \maxs=1
       \end{array}\right.$};
\node[anchor=west] (a2) at (3,0.3) {\scriptsize$\maxc=2$};
\end{scope}
\begin{scope}[xshift=0.6cm]
\node[circle,draw=brown,fill=brown,minimum size=0pt,inner sep=0pt] (c11) at (1.7,0.15) {\color{brown}.};
\path
    (c11) edge [loop] node {} (c11);
\end{scope}
\end{tikzpicture}}
\\
\bottomrule
\end{tabular}
\caption{Simplified versions of Conj.~(\ref{ex:conjecture4}) wrt conditions \ding{177} and \ding{178}}
\label{tab:conj4} 
\end{table*}

\section{Proving a Conjecture on {\sc Rooted Trees} Found Using Decomposition Methods}\label{sec:tree}

We give a conjecture about the {\sc Rooted Tree} combinatorial object.
As this proof will also use the {\sc Partition} combinatorial object, we start by introducing the object {\sc Partition}.

Consider the {\sc Partition} combinatorial object, where each instance is a sequence $\mathcal{S}$ of integers identified by the following characteristics:
$\pn$ the sequence length,
$\nval$ the number of distinct values in $\mathcal{S}$,
$\pmin$ (resp.~$\pmax$) the number of occurrences of the least (resp.~most) frequent integer in $\mathcal{S}$,
$\prange$ the difference $\pmax-\pmin$.

\begin{example}
We have the {\sc Partition} instance 
$\mathcal{S}=[1,1,1,1,1,1,2,2,3,3,4]$,
where  the values of  characteristics are $\pn=11$, $\nval=4, \pmin = 1, \pmax=6, \prange = 5$.
\end{example}

A {\sc Rooted Tree} is an acyclic digraph where the unique vertex that has no outgoing arc is called the \emph{root}, and where a \emph{leaf} is a vertex with no incoming arc. An instance of a rooted tree $\mathcal{T}$ is identified by the following characteristics:
$\tn$ the number of vertices of $\mathcal{T}$,
$\tleaves$ the number of leaves of $\mathcal{T}$,
$\tmindegree$ (resp.~$\tmaxdegree$) the minimum (resp.~maximum) number of successors of the vertices of $\mathcal{T}$.
Conj.~(\ref{eq:conjecture5}) gives a
sharp lower and upper bound on $\tleaves$ wrt $\tn$, $\tmindegree$, and $\tmaxdegree$.
\begin{equation}\label{eq:conjecture5}
\left\{~
\begin{aligned}
  \tmindegree=0\Rightarrow~ & \tleaves=1 \\
  \tmindegree>0\Rightarrow~ & \tleaves\in\left[
\left\lceil \frac{\tn\cdot\tmindegree+\tmaxdegree-\tmindegree-\tn+1}{\tmindegree}\right\rceil,\right.
\left.\left\lfloor \frac{\tn\cdot\tmaxdegree+\tmindegree-\tmaxdegree-\tn+1}{\tmaxdegree}\right\rfloor\right]
\end{aligned}
\right.
\end{equation}

\begin{proof}[Proof of Conj.~(\ref{eq:conjecture5})]
The first case is obvious as $\tmindegree =0$ means that the tree has just one vertex.

To prove the second case $\tmindegree > 0$, we first prove two conjectures of the {\sc Partition} object.
We then represent a {\sc Rooted Tree} as a {\sc Partition} object and finally derive from that representation the Conj.~(\ref{eq:conjecture5}).

For the {\sc Partition} object, let $o_i$ be the number of occurrences of the $i$\nobreakdash-th integer in $\mathcal{S}$, and $\pn_p$ the length of $\mathcal{S}$.
We have 
$\pn_p = \pmax +\sum_{i=1}^{\nval-1}o_i = \prange + \pmin + \sum_{i=1}^{\nval-1}o_i \geq \prange + \nval\cdot\pmin$. 
And as $\nval \in \mathbb{N}$, we obtain the inequality
 $\nval\leq${$ \left\lfloor\frac{\pn_p-\prange}{\pmin}\right\rfloor$}, the first true conjecture found in the context of the {\sc Partition} object. We also have
$\pn_p = \pmin + \sum_{i=2}^{\nval} o_i = \pmax -\prange  + \sum_{i=2}^{\nval}o_i \leq \nval\cdot\pmax-\prange$, which leads to the second true conjecture of the {\sc Partition} object $\nval\geq${$ \left\lceil (\pn_p+\prange)/\pmax\right\rceil$}.

Now we represent a {\sc Rooted Tree} as a vector $\mathcal{S}'$ of $\tn-1$ components where $\forall j\in [1;\tn-1]$, the $j$-th component $\mathcal{S}'[j]$ is the father of the vertex $j$ when the root is the vertex $j=0$. This configuration shows that a {\sc Rooted Tree} is a {\sc Partition} object where $\nval = \tn-\tleaves, \pn_p = \tn-1, \pmin = \tmindegree$ and $\pmax = \tmaxdegree$. As an illustration, the following {\sc Rooted Tree} 
$$
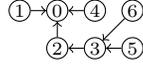
\tikz{
\node[circle,draw,minimum size=8pt,inner sep=0pt] (a) at (0.0,0) {\scriptsize 1};
\node[circle,draw,minimum size=8pt,inner sep=0pt] (b) at (0.5,0) {\scriptsize 0};\node[circle,draw,minimum size=8pt,inner sep=0pt] (c) at (1,0) {\scriptsize 4};\node[circle,draw,minimum size=8pt,inner sep=0pt] (d) at (0.5,-0.5) {\scriptsize 2};
\node[circle,draw,minimum size=8pt,inner sep=0pt] (e) at (1,-0.5) {\scriptsize 3};\node[circle,draw,minimum size=8pt,inner sep=0pt] (f) at (1.5,-0.5) {\scriptsize 5};\node[circle,draw,minimum size=8pt,inner sep=0pt] (g) at (1.5,0) {\scriptsize 6};
\draw[<-] (b) edge (a);
\draw[<-] (b) edge (c);
\draw[<-] (b) edge (d);
\draw[<-] (d) edge (e);
\draw[<-] (e) edge (f);
\draw[<-] (e) edge (g);
}$$ is represented by the {\sc Partition} object $\mathcal{S}'=[0,0,2,0,3,3]$ as we have $\mathcal{S}'[1]=\mathcal{S}'[2]=\mathcal{S}'[4] = 0, \mathcal{S}'[3] =2, \mathcal{S}'[5] =  \mathcal{S}'[6] = 3$.
So by applying the two bounds of $\nval$ which was proven earlier, we have

\begin{equation*}
\left\lceil(\pn_p+\prange)/ \pmax\right\rceil\leq\nval\leq\left\lfloor(\pn_p-\prange)/\pmin\right\rfloor\Leftrightarrow
\end{equation*}

\begin{equation*}
\left\lceil (\tn-1+(\tmaxdegree-\tmindegree))/\tmaxdegree\right\rceil\leq\tn-\tleaves\leq\left\lfloor(\tn-1-(\tmaxdegree-\tmindegree))/\tmindegree\right\rfloor\Leftrightarrow
\end{equation*}

\begin{equation*}
\tn - \left\lfloor(\tn-1-(\tmaxdegree-\tmindegree))/\tmindegree\right\rfloor\leq\hspace*{7pt}\tleaves\hspace*{7pt}\leq\tn - \left\lceil(\tn-1+(\tmaxdegree-\tmindegree))/\tmaxdegree\right\rceil\Leftrightarrow
\end{equation*}

\begin{equation*}
\left\lceil\tn -(\tn-1-(\tmaxdegree-\tmindegree))/\tmindegree\right\rceil\leq\hspace*{4pt}\tleaves\hspace*{4pt}\leq \left\lfloor\tn -(\tn-1+(\tmaxdegree-\tmindegree))/\tmaxdegree\right\rfloor
\end{equation*}

\noindent which leads to the Conj.~(\ref{eq:conjecture5}).
\end{proof}

\section{Conclusion}

In this report, five conjectures found by the decomposition methods introduced in~\cite{DecompGindullinBCDQ24} were proved. Namely, four conjectures on {\sc Digraphs} and one on {\sc Rooted Trees}. To deal with the complexity of those conjectures, various methods were used. The method used to prove conjectures on {\sc Digraphs} is enumerating all the corresponding disjoint and simple cases of the conjectures and prove each simple case by contradiction or deduction. The  conjecture on {\sc Rooted Trees} was proved by using some properties of partitioned sets.


\begin{thebibliography}{1}
\bibitem{DecompGindullinBCDQ24}
Ramiz Gindullin,
Nicolas Beldiceanu,
Jovial Cheukam{-}Ngouonou,
R{\'{e}}mi Douence, and
Claude-Guy Quimper.
Composing Biases by Using CP to Decompose Minimal Functional Dependencies for Acquiring Complex Formulae.
\emph{Proceedings of the 38th Annual AAAI Conference on Artificial Intelligence},
2024.
\end{thebibliography}
\end{document}